\documentclass[twoside,11pt]{article}

\usepackage{blindtext}

\newcommand{\ba}{\begin{array}}
\newcommand{\ea}{\end{array}}

\newcommand{\bit}{\begin{itemize}}
\newcommand{\eit}{\end{itemize}}
\newcommand{\be}{\begin{equation}}
\newcommand{\ee}{\end{equation}}
\newcommand{\bea}{\begin{eqnarray}}
\newcommand{\eea}{\end{eqnarray}}

\newcommand{\Rmn}[1]{\uppercase\expandafter{\romannumeral#1}}

\newcommand{\R}{\mathbb{R}}

\newcommand{\iprod}[2]{\left \langle #1, #2 \right \rangle }

\usepackage{algorithm}
\usepackage{algorithmic}

\usepackage{lipsum}
\usepackage{clrscode}
\usepackage{appendix}
\usepackage{bm}
\usepackage{booktabs}
\usepackage{url}
\usepackage{multirow}
\usepackage{textcomp}
\usepackage{amsmath}
\usepackage{amsfonts}
\usepackage{amssymb}
\usepackage{mathrsfs}
\usepackage{graphicx,graphics,subfigure}
\usepackage{epsf,epstopdf}
\usepackage{bbm}
\usepackage{dsfont}
\usepackage{indentfirst}
\usepackage{pdfpages}
\usepackage{pdflscape}
\usepackage{longtable}

\usepackage{jmlr2e}

\usepackage{lastpage}

\ShortHeadings{AdaFish: Fast low-rank parameter-efficient fine-tuning by using second-order information}{Hu and Li}
\firstpageno{1}

\begin{document}

\title{AdaFish: Fast low-rank parameter-efficient fine-tuning \\
by using second-order information}

\author{\name Jiang Hu \email hujiangopt@gmail.com \\
       \addr Massachusetts General Hospital and Harvard Medical School\\
       Boston, MA 02114, USA
       \AND
       \name Quanzheng Li \email li.quanzheng@mgh.harvard.edu \\
       \addr Massachusetts General Hospital and Harvard Medical School\\
       Boston, MA 02114, USA
       }

\editor{xx}

\maketitle

\begin{abstract}
Recent advancements in large-scale pretrained models have significantly improved performance across a variety of tasks in natural language processing and computer vision. However, the extensive number of parameters in these models necessitates substantial memory and computational resources for full training. To adapt these models for downstream tasks or specific application-oriented datasets, parameter-efficient fine-tuning methods leveraging pretrained parameters have gained considerable attention. However, it can still be time-consuming due to lots of parameters and epochs. In this work, we introduce AdaFish, an efficient algorithm of the second-order type designed to expedite the training process within low-rank decomposition-based fine-tuning frameworks. Our key observation is that the associated generalized Fisher information matrix is either low-rank or extremely small-scaled. Such a generalized Fisher information matrix is shown to be equivalent to the Hessian matrix. Moreover, we prove the global convergence of AdaFish, along with its iteration/oracle complexity. Numerical experiments show that our algorithm is quite competitive with the state-of-the-art AdamW method. 
\end{abstract}

\section{Introduction} 
Large-scale models have achieved significant success in diverse applications spanning computer vision and natural language processing. The efficacy of these models is largely attributed to their intelligently designed network architectures and the extensive number of parameters, enhancing their representational capabilities. When tailoring these models to specific downstream tasks using domain-specific datasets, relying solely on pretrained weights may not suffice, and fine-tuning is often employed to boost performance. However, fine tuning all the weights still requires huge resources of the storage and the computing. This challenge has spurred research into parameter-efficient fine-tuning, a strategy aimed at reducing the number of weights to be adjusted by confining training to lower-dimensional subspaces rather than the entire parameter space. The low-rank adaptation (LoRA) in \citep{hu2021lora}  assumes the weight update in the fine-tuning stage is low-rank and directly performs the update on the pretrained weight matrices in low-rank subspaces. By reorganizing the weight matrices into a single third-order tensor, FacT \citep{jie2023fact} utilizes the tensor low-rank Tucker decomposition and tensor train decomposition to further reduce the number of parameters and yield superior performance over LoRA. The tensor CANDECOMP/PARAFAC (CP) decomposition \citep{kolda2009tensor} is further investigated in \citep{wang2023parameter} to establish a large-scale multimodal foundation model.

Among the training strategies of the parameter-efficient fine-tuning frameworks, the most popular one is the AdamW method \citep{loshchilov2017decoupled}, which introduces an additional weight decay on the well-known Adam optimizer \citep{kingma2014adam}. Moving beyond first-order information, some studies have focused on leveraging approximate Hessian information for optimization. AdaHessian \citep{yao2021adahessian} utilizes a Hutchinson-based method to approximate the diagonal part of the Hessian. Subsequently, this idea was generalized in \citep{liu2023sophia}, where a more computable and high-quality diagonal approximation is presented. Another approach to connect with the Hessian is through the Fisher information matrix. The KFAC method in \citep{martens2015optimizing} employs a Kronecker-factored approximation to the Fisher information matrix for the efficient calculation of matrix inverses involved in the natural gradient direction. This approach has been extended to support a broader range of neural networks, including convolutional neural network \citep{grosse2016kronecker} and recurrent neural network \citep{martens2018kronecker}. Since the classical Fisher information matrix is defined based on the vectorized formulation of the weight matrix, the inverses in Kronecker-factor approximation could still be computationally costly. To address this issue, the concept of a generalized Fisher information matrix, based on the matrix format (i.e., natural format) of the weight matrices, is introduced in \citep{yang2022efficient}. Only a single smaller matrix inverse is needed in computing the associated natural gradient direction. All of these methods are designed for general purposes. The motivation behind this paper is to design an efficient second-order type algorithm that capitalizes on the inherent low-rank structure in parameter-efficient fine-tuning.

\subsection{Contributions}
We summarize our contributions as follows. 
\begin{enumerate}
    \item We leverage the inherent low-rank properties of weight matrices in fine-tuning models to introduce a novel approach for approximating Hessian information using a portable Fisher information matrix. Under certain conditions, we establish the equivalence of this Fisher information matrix with the Hessian matrix, highlighting its utility in capturing second-order information efficiently.

    \item We present an adaptive Fisher method, dubbed as AdaFish, the efficient fine-tuning of LoRA-based models. Combining the popular exponential moving average with accessible storage and computation, we construct an adaptive Fisher information matrix, which could serve as an efficient alternative to the second-order momentum in AdamW. Unlike conventional approaches, this Fisher information matrix is neither purely diagonal nor element-wise. The comparison of our AdaFish with existing methods is presented in Table \ref{tab:optimizers}.  
    
    \item The convergence and iteration/oracle complexity of AdaFish are established. Through empirical evaluations on both image classification and language processing tasks, we demonstrate that AdaFish not only achieves superior performance but also reduces the required number of epochs by half compared to AdamW, underscoring its efficiency and effectiveness in model fine-tuning. 
\end{enumerate}

\subsection{Notations}
For any matrix $U \in \R^{m\times n}$, we denote its Euclidean norm by $\|U\|$. For any two matrices $U$ and $V$ of $\R^{m\times n}$, we use $\odot$ to denote the Hadamard product, i.e., $(U \odot V)_{ij} = U_{ij}V_{ij}$. For two matrices $A \in \R^{m_1 \times n_1}, B \in \R^{m_2 \times n_2}$, we use $\otimes$ to denote the Kronecker product, i.e., 
 \[ W=A\otimes B =  \begin{pmatrix}
 	a_{1,1} B &\ldots &a_{1,n_1}B \\
 	\vdots & \ddots & \vdots \\
 	a_{m_1, 1} B & \ldots & a_{m_1, n_1} B.\\
 \end{pmatrix}\]
 Given a set $S$, $|S|$ represents the cardinality of $S$. For a differentiable function $f(U,V)$, the notations $\nabla_U f(U,V), \nabla_V f(U,V)$, and $\nabla f(U,V)$ denote the partial gradient with respect to $U$, the partial gradient with respect to $V$, and the gradient with respect to both $U$ and $V$, respectively. The Hessian of $f$ is represented by ${\rm Hess} f$. 

\begin{table*}[h]
\caption{Comparison of different optimizers and their associated first-order $\tilde{m}_t$ and second-order momentum $\tilde{v}_t$. Here, we use $g_t$ to denote the stochastic gradient used in the $t$-th iteration. The last column represents whether the second-order momentum is nondiagonal. The matrix $D^{(s)}$ in AdaHessian is diagonal, which is constructed by using the Hutchinson’s method to approximate Hessian diagonal. In Sophia, $\hat{g}$ is calculated based on a resampled loss function.}
\centering
\footnotesize
\setlength{\tabcolsep}{2pt}
\begin{tabular}{|l|c|c|c|}
\hline
Optimizer & $\tilde{m}_t$ & $\tilde{v}_t$ & Nondiagonal app. \\
\hline
SGD \citep{robbins1951stochastic} & $\beta_1 \tilde{m}_{t-1} + (1 - \beta_1)g_t$ & $1$ & \texttimes
\\ \hline
AdamW \citep{loshchilov2017decoupled} & $\frac{(1-\beta_1) \sum_{i=1}^t \beta_1^{t-i}g_i}{1-\beta_1^t}$ & $\sqrt{\frac{(1-\beta_2) \sum_{i=1}^t \beta_2^{t-i}g_i \odot g_i}{1-\beta_1^t}}$ & \texttimes \\ \hline
AdaHessian \citep{yao2021adahessian} & $\frac{(1-\beta_1) \sum_{i=1}^t \beta_1^{t-i}g_i}{1-\beta_1^t}$ & $\sqrt{\frac{(1-\beta_2) \sum_{i=1}^t \beta_2^{t-i}D^{(s)}_i D^{(s)}_i}{1-\beta_2^t}}$ & \texttimes \\ \hline
Sophia \citep{liu2023sophia} & $\beta_1 \tilde{m}_{t-1} + (1 - \beta_1)g_t$ & $\beta_2 \tilde{v}_{t-1} + (1-\beta_2) B_t \cdot \hat{g}_t \odot \hat{g}_t$ & \texttimes\\ \hline
\textbf{AdaFish (Ours)} & $\frac{(1-\beta_1) \sum_{i=1}^t \beta_1^{t-i}g_i}{1-\beta_1^t}$ & $\frac{(1-\beta_2) \sum_{i=1}^t \beta_2^{t-i}g_ig_i^\top}{1-\beta_2^t}$ & \checkmark \\
\hline
\end{tabular}
\label{tab:optimizers}
\end{table*}

\section{Preliminaries}
In this section, we will review the popular low-rank parameter-efficient fine-tuning methods and Fisher information matrix-based stochastic optimization methods.  
\subsection{LoRA}
The most popular low-rank parameter-efficient fine-tuning method is the low-rank adaption, abbreviated as LoRA  \citep{hu2021lora}. In neural networks, the dense layers perform matrix multiplication, and the weight matrices in these layers usually have a full rank. However, when adapting to a specific task, pre-trained language models have been shown to have a low intrinsic dimension, allowing them to learn efficiently even with a random projection to a smaller subspace. Building on this idea, they hypothesize that the updates to the weights during adaptation also have a low intrinsic rank. To constrain the update of a pre-trained weight matrix, the update is performed in a low-rank space, namely,
\[  W = W_0 + \Delta W = W_0 + U^\top V, \]
where $W \in \R^{d\times k}, U \in \R^{r \times d}, V \in \R^{r \times k}$, and the rank $r \ll \min(d,k)$. Due to the low rankness, the computation of matrix-vector multiplication is reduced from $\mathcal{O}(dk)$ of $\Delta W x$ to $\mathcal{O}(r(d+k))$ of $UV^\top x$. During training, $W_0$ is frozen and does not receive gradient updates, while $U$ and $V$ are updated. For the initialization, they use a random Gaussian initialization for $V$ and set $U$ as zero. In practice, $r$ is often set to $4$.

\subsection{Parameter-efficient fine-tuning based on tensor low-rank decompositions}
In LoRA, since the weight matrices are decomposed via a matrix wise format, it can only reduce the intra-weight redundancy. To alleviate the inter-weight redundancy, the tensor low-rank decomposition based methods are developed.

The tensorized representation of weight matrices facilitates the application of sophisticated tensor decomposition techniques, such as tensor train \citep{oseledets2011tensor} and tensor Tucker decompositions \citep{de2000multilinear}, to fine-tune weights in a compressed format. For instance, the Tucker decomposition represents a low-rank 3-order tensor $\Delta \mathcal{W}^{\rm F} \in \R^{(12 L) \times d \times d}$ using a core tensor $C^{\rm F} \in \R^ {r \times r \times r}$ and three factor matrices $P^{\rm F} \in \R^{r \times 12L}$,  $A^{\rm F} \in \R^{r\times d}$ and $B^{\rm F} \in \R^{r\times d}$, scaled by a scalar $s \in \R$, namely,
\be \label{eq:Tucker} \Delta \mathcal{W}^{\rm F} = s \cdot C^{\rm F} \times_1 (P^{\rm F})^\top  \times_2 (A^{\rm F})^\top \times_3 (B^{\rm F})^\top,  \ee
where $\times_i$ is mode-$i$ product, i.e.,
\[ \begin{aligned}
	\Delta \mathcal{W}_{i,j,k}^{\rm F} & = s \cdot \sum_{t_1 =1}^r \sum_{t_2=1}^r \sum_{t_3 = 1}^r C_{t_1,t_2,t_3}^{\rm F} P_{t_1, i}^{\rm F} A_{t_2,j}^{\rm F} B_{t_3,k}^{\rm F}, \\
	\forall & 1\leq i \leq 12L, 1\leq j,k \leq d. 
\end{aligned}\]
Then, the fine-tuning is performed by training the factors, namely,
\be \label{eq:fact}  \mathcal{W} = \mathcal{W}_0 + \Delta \mathcal{W}^{\rm F}, \ee
where $\mathcal{W}_0$ is the pretrained parameters and frozen. The corresponding total number of tuning parameters is reduced to $r^3 + (12L + 2d)r$. When $r \ll \min\{L, d\}$,  the order is $\mathcal{O}(\max\{d,L\} r)$ instead of $\mathcal{O}(Ldr)$ in LoRA. Such reduction leads to better scalability of tensor-based decompositions.

More recently, the tensor CP decomoposition is also investigated in \citep{wang2023parameter} to fine tune the large-scale multimodal foundation model. The reparameterization on the weight update is given by 
\be \label{eq:CP}  \Delta \mathcal{W}^{\rm C} = \sum_{t=1}^r \lambda_t P^{\rm C}_t \circ A^{\rm C}_t \circ B^{\rm C}_t, \ee
where $\lambda \in \R^r, P^{\rm C} \in \R^{r \times 12L}$,  $A^{\rm C} \in \R^{r \times d}$, and $B^{\rm C} \in \R^{r \times d}$, and the corresponding elementwise formula is 
\[ \Delta \mathcal{W}_{i,j,k}^{\rm C} = \sum_{t=1}^r \lambda_{t} P^{\rm C}_{i,t} 
 A^{\rm C}_{j,t} B^{\rm C}_{k,t}. \]
The size of the core tensor is diagonal and its diagonal elements are represented by the vector $\lambda$. The total number of tuning parameters is $(12L + 2d +1)r$. Compared with Tucker decomposition, CP decomposition further reduces the number of parameters in the core tensor.

\subsection{Fisher information matrix based stochastic optimization methods} \label{sec:fisher-pre}
In the realm of probabilistic models, the Fisher information matrix was introduced by \citep{amari1985differential} as a Riemannian metric on the parameter space. This metric defines the natural gradient direction, characterizing the steepest descent in distribution space. 
Natural gradient-based stochastic methods have attracted significant attention in deep learning, as seen in \citep{martens2015optimizing,grosse2016kronecker,martens2018kronecker,martens2020new,goldfarb2020practical}. The efficiency of those methods relies on the fast inverse process of the Fisher information matrix or its approximation. Consider the following optimization problem with loss in the form of negative log-probability
\be \label{prob:fim} \min_{U \in \R^{m\times n}} \;\; f(U):=-\frac{1}{|S|}\sum_{(x,y) \in S}\sum \log p(y|x, U), \ee
where $x$ is the input, $y$ is the label, $S$ denotes the dataset, and $U$ represents the parameters. Note that many popular loss functions \citep{amari1985differential,martens2020new}, including the mean squared error and the cross entropy, yield the negative log-probability form. Denote by $u \in \R^{mn}$ the vectorization of $U \in \R^{m\times n}$, i.e., $u = {\rm vec}(U)$. The Fisher information matrix of \eqref{prob:fim} is given by 
\[ F(u) = \mathbb{E}_{P_x} \left[ \mathbb{E}_{P_{y|x,u}}\left[ \nabla \log p(y|x,u) \nabla \log p(y|x,u)^\top \right ] \right],  \]
which is an $mn$-by-$mn$ matrix. Since the expectations over $P_{x}, P_{y|x,u}$ may not have closed-form expressions, the following empirical Fisher information is constructed:
\[ \bar{F}(u) = \frac{1}{|S|} \sum_{(x,y) \in S} \left[ \nabla \log p(y|x,u) \nabla \log p(y|x,u)^\top \right]. \]
Note that the empirical Fisher information can be defined without requiring the loss function to be of negative probability form. 

After having the Fisher information matrix, the natural gradient direction 
\[ F^{-1}(u) \nabla f(u)\]
or its stochastic approximation is shown to be a better descent direction. The inverse of $F(u)$ or $\bar{F}(u)$ will be computationally costly if the total number of parameters, i.e., $mn$, is large. To address this, the authors \citep{martens2015optimizing} investigate the Kronecker-factored approximation of $F(u)$ and reduce the computations to the inverses of two smaller matrices of $m \times m$ and $n \times n$. The corresponding computational cost is still much higher than the stochastic gradient-type methods. The approaches designed therein include lazy updates on the Kronecker-factored Fisher information matrix and block diagonalization.

\section{An efficient second-order algorithm for fast training}
Consider a neural network model with pretrained weight $W_0 \in \R^{n\times k}$. Let $\mathcal{D}$ be the dataset and $f$ be the loss function of the downstream task. When LoRA is utilized for the parameter-efficient fine-tuning, the corresponding optimization problem is
\be \label{prob:LoRA} \min_{U \in \R^{r\times n}, V \in \R^{r \times k}} \quad \mathbb{E}_{\xi \sim \mathcal{D}}[f(W_0 + U^\top V; \xi)]. \ee
Note that where $r$ is much smaller than $\min\{n,k\}$ owing to the setting of low-rank fine-tuning. 

\subsection{Generalized Fisher information matrix}
Assume that the loss function $f$ is of log-probability form. As introduced in subsection \ref{sec:fisher-pre}, the dimension of the classic Fisher information matrix for $U$ ($V$) is ${nr\times nr}$ (${kr\times kr}$). Due to the large dimensionality, computing the corresponding natural gradient direction becomes computationally expensive, even when employing the Kronecker-factored approximation, which still requires the inversion of an $n \times n$ (or $k \times k$) matrix alongside an $r \times r$ matrix.  

We now introduce the generalized Fisher information matrix of \eqref{prob:LoRA}. By using the matrix calculus, we have the gradients of $f$ with respect to $U$ and $V$
\[ \begin{aligned}
    \nabla_U f(W_0 + UV^\top; \xi) & = V \nabla f(W_0 + UV^\top; \xi)^\top \in \R^{r\times n}, \\
    \nabla_V f(W_0 + UV^\top; \xi) & = U \nabla f(W_0 + UV^\top; \xi) \in \R^{r \times k}. \\
\end{aligned} \]
Building upon \citep{yang2022efficient} and considering $U$ and $V$ separately, we define the following generalized Fisher information matrices for $U$
\[ \begin{aligned}
    F^L(U) = \mathbb{E}_{\xi} & \left[ \nabla_U f(W_0+UV^\top; \xi) \right. \\
    & \left. \nabla_U f(W_0 + UV^\top; \xi)^\top \right] \in \R^{r\times r}, \\ 
    F^R(U) = \mathbb{E}_{\xi} & \left[ \nabla_U f(W_0+UV^\top; \xi)^\top \right. \\
    & \left. \nabla_U f(W_0 + UV^\top; \xi) \right] \in \R^{n \times n}.
\end{aligned}
\]
We call $F^L(U)$ and $F^R(U)$ the left and right generalized Fisher information matrices, respectively. 
Note that while $F^R(U)$ possesses a larger dimensional space compared to $F^L(U)$, it exhibits a low-rank characteristic, with its rank not exceeding $r$. Analogously, one can define a generalized Fisher information matrix for the variable $V$. It is worth noting that such low-rankness and small scale of the generalized Fisher information are unique to the low-rank parameter-efficient fine-tuning setting. 

\subsection{Connection with Hessian matrix}
Let ${\theta}$ denote the parameter $U$ or $V$. Assume that the loss function is of negative log-probability form, i.e., for some probability $p$, 
\[ f(\theta; x,y) :=-\log p(y|x,\theta). \] We then have the following connections among the matrix Fisher, the classic Fisher and the Hessian matrix. 
\begin{lemma} \label{lem:equiv}
    Assume that each column of the sample gradient $g = g(\theta;x,y) \in \mathbb{R}^{r \times n}$ is independent and identically distributed random vector with zero mean under the distribution $p(y \mid x, \theta)$ for any $\theta$. We have:
    \be \label{eq:mat-vec-Fisher}
    \mathbb{E}_{\xi}\left[\operatorname{vec}(g) \operatorname{vec}(g)^{\top}\right]=\mathbb{E}_{\xi}\left[I_n \otimes\left(\frac{1}{n} g g^{\top}\right)\right]. 
    \ee
    In addition, it holds:
    \be \label{eq:matFish-Hess}
    \mathbb{E}_{\xi}\left[I_n \otimes \frac{1}{n} g g^{\top} \right] = {\rm Hess} \mathbb{E}_{\xi} [f({\rm vec}(\theta))].   \ee
\end{lemma}
\begin{proof}
The equation \eqref{eq:mat-vec-Fisher} can be shown by following \citep{yang2022efficient}. By noting the fact that $ \mathbb{E}_{\xi}\left[{\rm vec}(g) {\rm vec}(g)^{\top}\right] = {\rm Hess \;} \mathbb{E}_{\xi} [f({\rm vec}(\theta))]$, 
we have \eqref{eq:matFish-Hess} holds. 
\end{proof}

The above lemma validates the use of the generalized Fisher information matrix through its connection to the Hessian matrix.

Given the computational challenge or impracticality of calculating the exact expectation $\mathbb{E}[\frac{1}{n} g g^\top ]$, we consider a computationally efficient empirical generalized Fisher matrix defined as:
\[ \bar{F}(\Theta) = I_n \otimes \frac{|\mathbf{B}|}{n} g(\mathbf{B}, \theta) g(\mathbf{B}, \theta)^\top, \]
where $g(\mathbf{B}, \theta):= -\frac{1}{|\mathbf{B}|} \sum_{(x,y) \in \mathbf{B}} \nabla \log p(y|x, \theta)$ denotes the mini-batch gradient over $\mathcal{B}$. 

\subsection{Natural gradient direction}

For the matrix of Kronecker product form, we have $(I_n \otimes A)^{-1} = I_n \otimes A^{-1}$ under the condition that $A$ is invertible. Moreover, for any matrix $D \in \R^{r \times n}$, it holds that
\[  \left[I_n \otimes A^{-1}\right] {\rm vec}(D) = {\rm vec}(A^{-1} D). \]
By the above formula and by substituting the exact expectation with its empirical court part, the corresponding natural gradient directions are defined as follows, where the generalized Fisher information matrix serves as a preconditioner:
\[ \begin{aligned}
    D^L & = (\bar{F}^L(U) + \lambda I)^{-1} g(U), \\ 
    D^R & = g(U) (\bar{F}^R(U) + \lambda I)^{-1}, 
\end{aligned}
\]
where $\bar{F}^L(U) := |\mathcal{D}| g(U) g(U)^\top$, \; $\bar{F}^R(U) := |\mathcal{D}| g(U)^\top g(U)$, $\lambda > 0$ is a small constant to ensure the positive definiteness of the matrices to be inverted, and $g(U)$ is the gradient of $\mathbb{E}_{\xi \sim \mathcal{D}}[f(W_0 + U^\top V; \xi)]$. It is obvious that $D^L$ can be obtained by inverting a small $r$-by-$r$ matrix. In terms of $D^R$, since $\bar{F}^R(U)$ is low-rank, we can apply the Sherman-Morrison-Woodbury formula to reformulate it as 
\[ \begin{aligned}
D^R(U) =& \lambda^{-1} g(U) - g(U)(I_r + g(U)^\top \\
& g(U) )^{-1} g(U)^\top g(U). \end{aligned} \]
Therefore, we only need to compute the inverse of a small $r$-by-$r$ matrix. In comparison to $D^R$, computing $D^L$ requires two additional matrix multiplications. Given the small scale of $r$, both directions can be computed efficiently and without much cost.

\subsection{Our algorithm: AdaFish}
It is known that the exponential moving average is a useful tool to construct momenta and boost the performances in stochastic optimization. Inspired by this, we focus on developing an adaptive Fisher information matrix for preconditioning. At $t$-th iteration, we first calculate a stochastic gradient $g_t$ from minibatch $\mathbf{B}_t$. Then, based on the properties of the Kronecker product, we have two generalized minibatch Fisher information matrices, 
\[ \bar{F}^L(\theta_t):= |\mathbf{B}_t| g_t g_t^\top \quad {\rm and} \quad \bar{F}^R(\theta_t):= |\mathbf{B}_t| g_t^\top g_t. \] 
Note that $\bar{F}^L(\theta_t)$ is an $r$-by-$r$ matrix, which can be cheaply stored. Thus, we utilize the exponential moving average on $\bar{F}^L(\theta_t)$ to construct an adaptive Fisher information matrix as
\[ h_t = \beta_2 h_{t-1} + (1-\beta_2) g_tg_t^\top, \]
where $h_{t-1}$ represents the adaptive Fisher information matrix used in the $(t-1)$-th iteration. Note that the constant $|\mathbf{B}_t|$ is ignored and we need an extra parameter to control the magnitude of $h_t$. 

By evoking the popular $\ell_2$-regularization weight decay and the construction of first-order momentum in AdamW \citep{loshchilov2018decoupled}, our AdaFish is presented in Algorithm \ref{alg:adafish}, where $\lambda, \gamma, \delta$ are parameters controlling the weight decay, the magnitude of $h_t$, and numerical stability. Compared to AdamW, the use of such a nondiagonal momentum  allows us capturing more Hessian information without introducing much additional computational cost. 
\begin{algorithm} \label{alg:adafish}
\caption{AdaFish for fast parameter-efficient fine-tuning}
\begin{algorithmic}
\REQUIRE {$\theta_0$, learning rate $\{\eta_t\}_{t=0}^{T}$, hyperparameters $\lambda, \gamma, \beta_1, \beta_2, \delta$.} 
Set $m_{-1} = 0$, $h_{-1} = 0$\;
\FOR{$t = 0$ to $T$}
    \STATE Compute $g_t = \nabla f_t(\theta_t)$\;
    \STATE $m_t = \beta_1 m_{t-1} + (1 - \beta_1)g_t$\;
    \STATE $h_t = \beta_2 h_{t-1} + (1-\beta_2)g_tg_t^\top$\;
    \STATE $\hat{m}_t = m_t/(1-\beta_1^t)$\;
    \STATE $\hat{h}_t = h_t/(1-\beta_2^t)$\;
    \STATE $\theta_t = \theta_t - \eta_t \lambda \theta_t$\;
    \STATE $\theta_{t+1} = \theta_t - \eta_t \cdot (\gamma \cdot \hat{h}_t + \delta)^{-1} \cdot \hat{m}_t$
\ENDFOR
\end{algorithmic}
\end{algorithm}

\subsection{Extension to low-rank tensor decomposition based fine-tuning}
Consider the low-rank tensor decomposition of the weight update $\Delta \mathcal{W} \in \R^{\ell\times n\times k}$, 
\[ \Delta \mathcal{W} = \mathcal{C} \times_2 U \times_3 V, \]
where $\mathcal{C} \in \R^{\ell \times r \times r}$, $U \in \R^{r \times n}$, and $V \in \R^{r\times k}$. The associated fine-tuning problem is:
\be \label{prob:tensor} \min_{\mathcal{C}, U, V} \quad \mathbb{E}_{\xi} [f(\mathcal{C}, U, V; \xi)]. \ee
Since $U$ and $V$ are in matrix format, we can define their corresponding generalized Fisher information matrix as in the LoRA setting. Regarding the tensor $\mathcal{C}$, one approach is to split it into $\ell$ matrices of size $r \times r$ and define the generalized Fisher information matrix correspondingly. Consequently, in each iteration, to form $g_tg_t^\top$, we need $(n+k + \ell r)r^2$ computation flops and $(\ell + 2) r^2$ storage. However, in LoRA, the total computation flops and storage are $\ell(n+k)r^2$ and $2\ell r^2$, respectively. 

\section{Convergence analysis}
To establish the convergence, we rely on the following assumptions on the loss function and the stochastics involved in Algorithm \ref{alg:adafish}.

\begin{assumption} \label{assum}
We assume that the loss function $f(\cdot)$ satisfies the following conditions.
\begin{enumerate}
    \item $f(\theta; \xi)$ is continuously differentiable and has a lower bound, i.e., there exists an $f^* > -\infty$ such that $f(\theta; \xi) \geq f^*$ for any $\theta$ and $\xi \in \mathcal{D}$. 
    
    \item The gradient $\nabla f(\theta; \xi)$ is $L_{f}$-Lipschitz continuous, i.e., $\| \nabla f\left(\theta_1, \xi \right)$ $-\nabla f\left(\theta_2, \xi \right)\left\| \leq L_{f}\right\| \theta_1-\theta_2 \|$ for any $\theta_1, \theta_2$ and $\xi \in \mathcal{D}$.
    
    \item The mini-batch gradient is unbiased, bounded, and with bounded variance, namely, there exist two positive constants $\sigma, D_g$, such that
    \be \label{eq:assum-grad} \left \{
    \begin{aligned}
    & \mathbb{E}[\nabla f_t(\theta_t)] = \nabla f(\theta_t), \\
    & \mathbb{E}[\| \nabla f_t(\theta_t) - \nabla f(\theta_t) \|] \leq \sigma, \\
    & \| \nabla f(\theta_t, \xi) \| \leq D_g, \;\; \forall \xi \in \mathcal{D}.   
    \end{aligned} \right. \ee
\end{enumerate}
\end{assumption}

These assumptions require that the loss function is sufficiently smooth and the stochastic gradient estimation is accurate enough. They are broadly used in the analysis of stochastic optimization methods, e.g., the stochastic gradient method and adaptive type methods. 

To establish the convergence, we define the dynamic function $f_t(\theta)$, which is a combination of the original loss function $f(\theta):= \mathbb{E}_{\xi} [f(\theta; \xi)]$ and a dynamic regularization $\frac{\lambda}{2}\| \theta \|_{v_t}^2$ with $v_t:= \gamma \hat{h}_t + \delta$ and $\| \theta \|_{v_t}:= \sqrt{ {\rm trace}(\theta^\top v_t \theta_t)}$, i.e., 
\[ \tilde{f}_t(\theta_t) = f(\theta) + \frac{\lambda}{2}\|\theta\|_{v_t}^2. \]
Based on this auxiliary function, we have the following convergence guarantee. 

\begin{theorem} \label{thm:global}
Suppose that Assumption \ref{assum} is satisfied and let $\{\theta_t\}$ be generated by Algorithm \ref{alg:adafish} with fixed step size, i.e., $\eta_t \equiv \eta$. By setting $\eta \leq \frac{\delta^{1.25} b \epsilon^2}{6\left(D_g^2+\delta\right)^{0.75} \sigma^2 L}$, $\beta_1 \geq 1- \frac{\delta^{0.5} b \epsilon^2}{3\left(D_g^2+\delta\right)^{0.5} \sigma^2}$ and $\beta_2 \in(0,1)$ for all iterations, after $T=\mathcal{O}\left(\max \left(\frac{D_g^{2.5} L \Delta \sigma^2}{\delta^{1.25} b \epsilon^4}, \frac{D_g^2 \sigma^4}{\delta b^2 \epsilon^4}\right)\right)$ iterations with $\Delta=f\left(\theta_0\right)-f^*$, we have
\be \label{eq:complexity}
\begin{aligned}
& \frac{1}{T} \sum_{t=0}^{T-1} \mathbb{E}\left[\left\|\nabla \tilde{f}_t\left(\theta_t\right)\right\|_2^2\right] \leq \epsilon^2, \\
& \frac{1}{T} \sum_{t=0}^{T-1} \mathbb{E}\left[\left\|\theta_t-\theta_{t+1}\right\|_{v_t}^2\right] \leq \frac{\eta^2 \epsilon^2}{4}. \\ 
\end{aligned}
\ee
Moreover, the total stochastic gradient complexity to achieve \eqref{eq:complexity} is $\mathcal{O}\left(\max \left(\frac{D_g^{2.5} L \Delta \sigma^2}{\delta^{1.25} \epsilon^4}, \frac{D_g^2 \sigma^4}{\delta b \epsilon^4}\right)\right)$.
\end{theorem}

The above theorem gives the $\mathcal{O}(\epsilon^{-4})$ complexity of the gradient norm and the difference between two consecutive iterations. It matches the result of AdamW given in \citep{zhou2022towards}. 

\begin{remark}
Theorem \ref{thm:global} assumes the use of constant step size. One can follow the analysis of AdamW in \citep{zhou2022towards} to establish similar complexity bounds.
\end{remark}

\begin{figure*}
    \centering
    \includegraphics[width=0.4\textwidth]{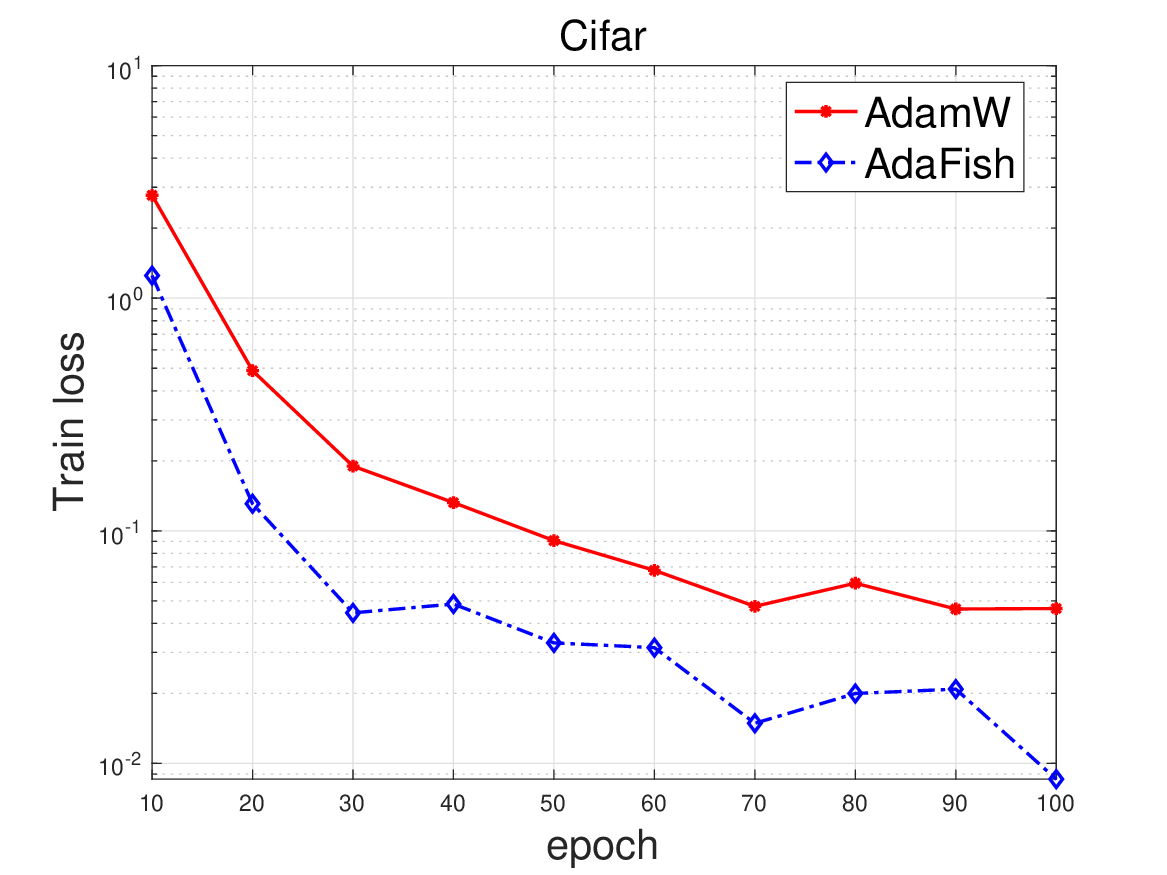}
    \includegraphics[width=0.4\textwidth]{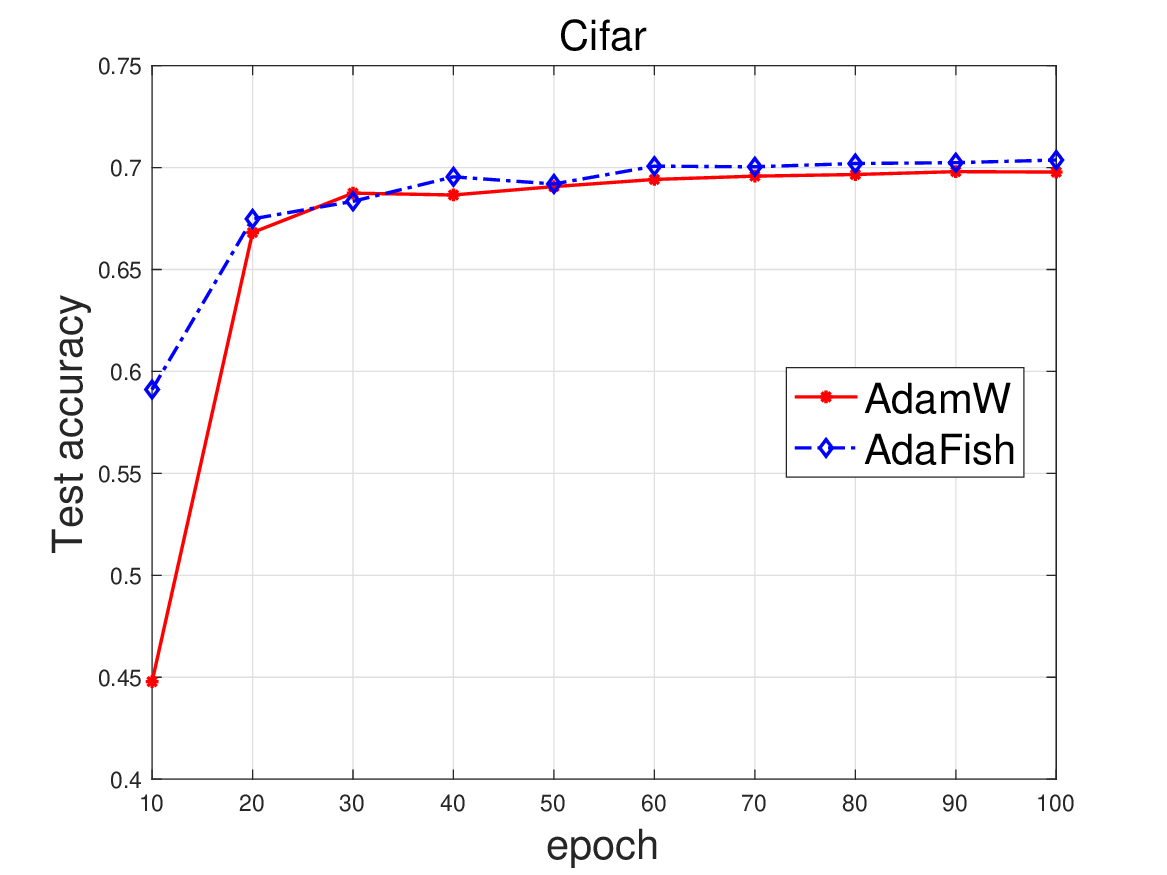}\\
    \vspace{0.2cm}
    \includegraphics[width=0.4\textwidth]{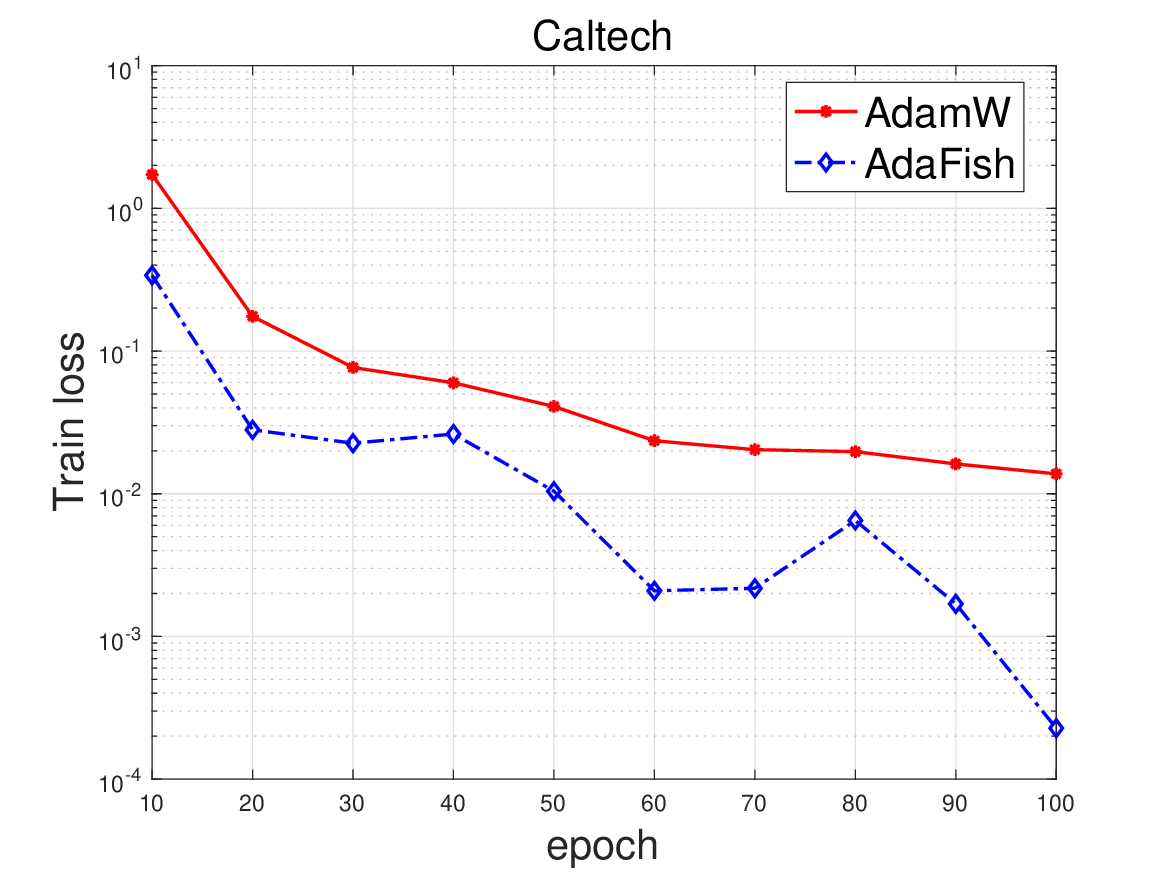}
    \includegraphics[width=0.4\textwidth]{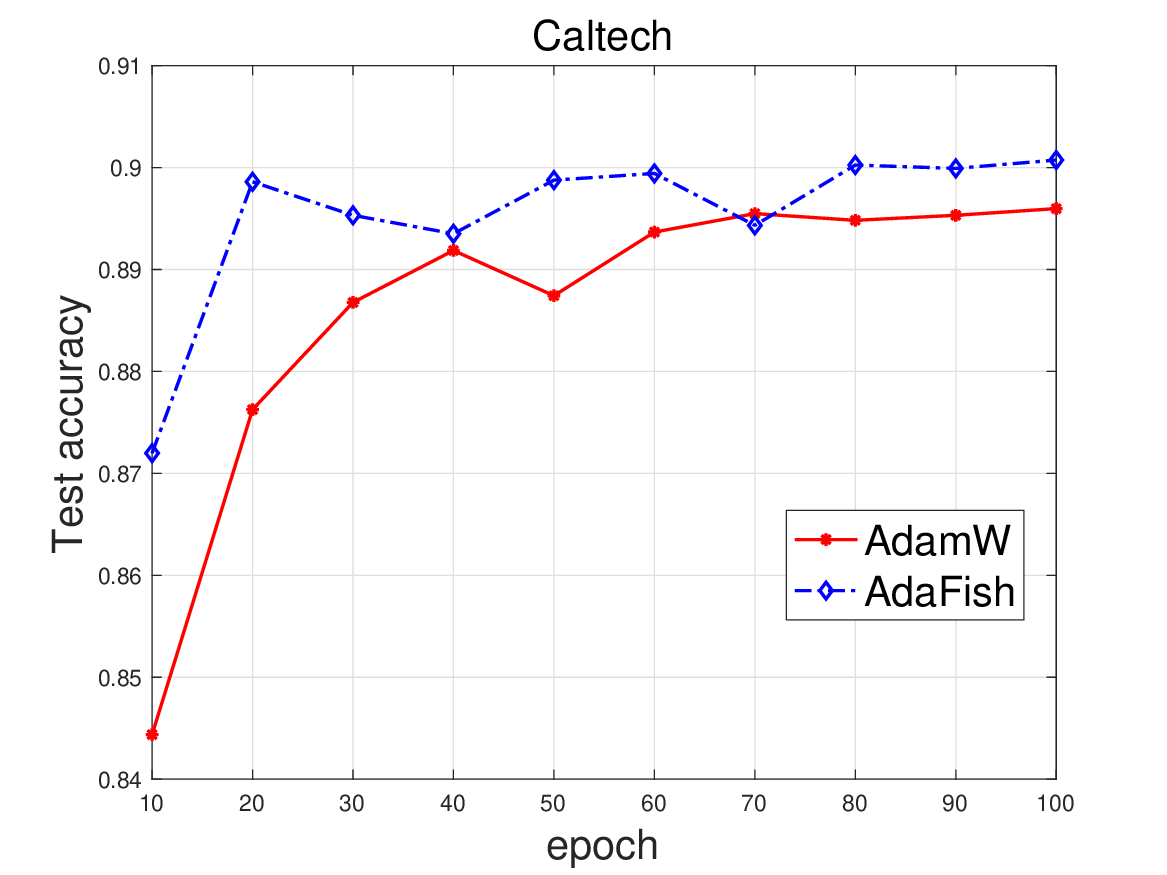}\\
    \vspace{0.2cm}
    \includegraphics[width=0.4\textwidth]{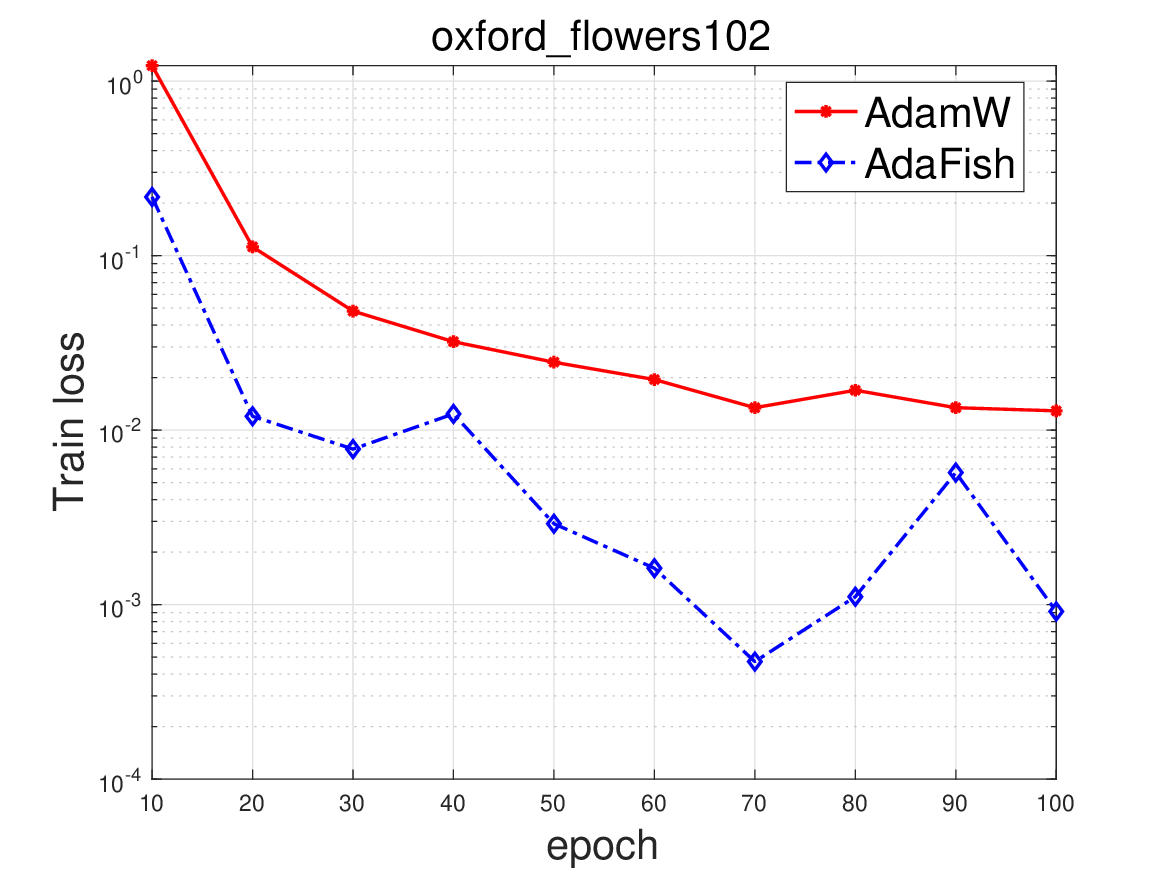}
    \includegraphics[width=0.4\textwidth]{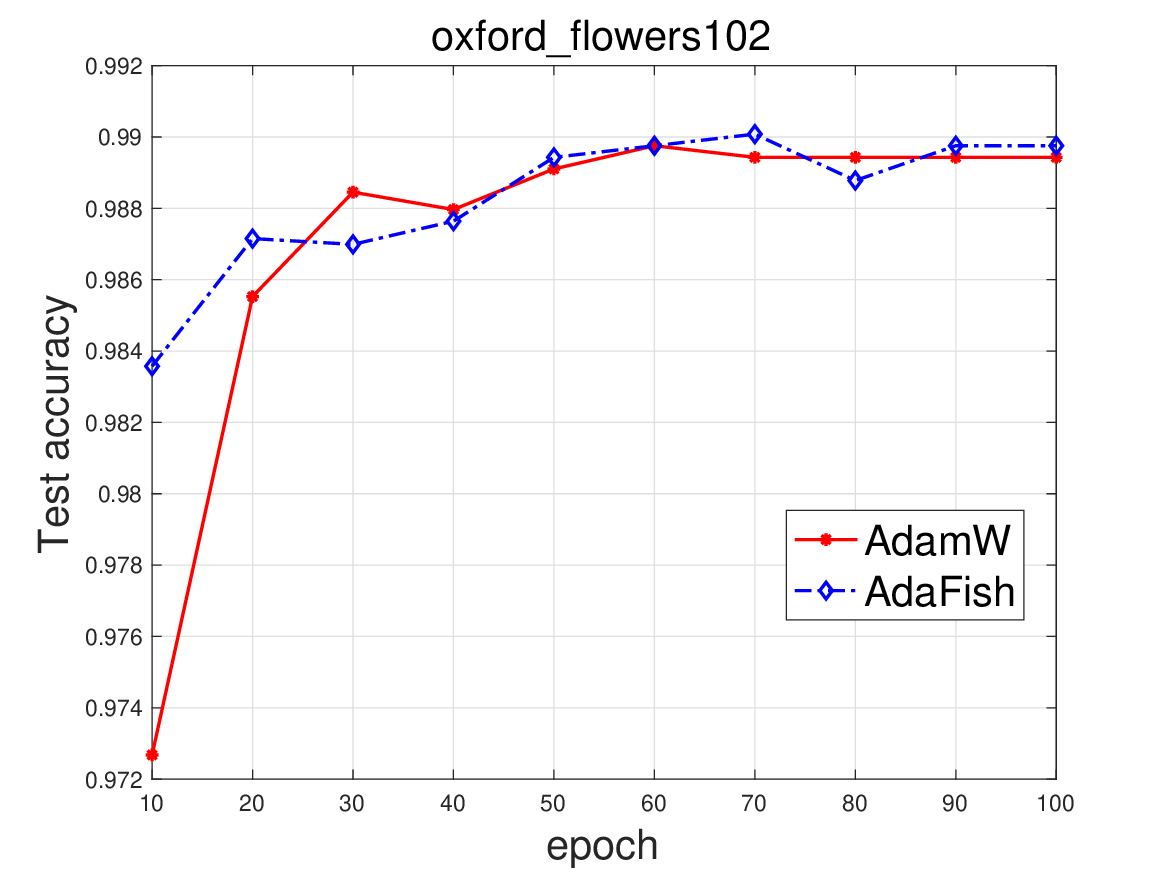} \\
    \vspace{0.2cm}
    \includegraphics[width=0.4\textwidth]{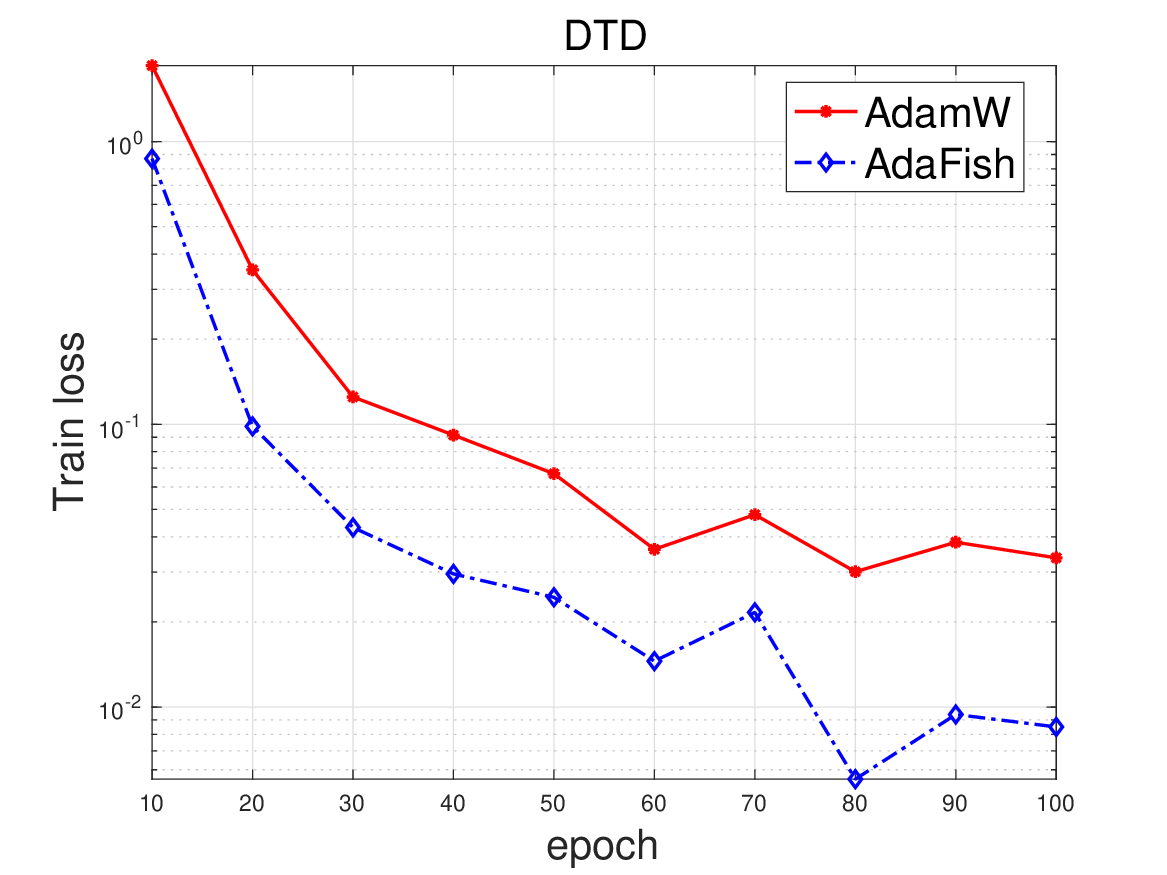}
    \includegraphics[width=0.4\textwidth]{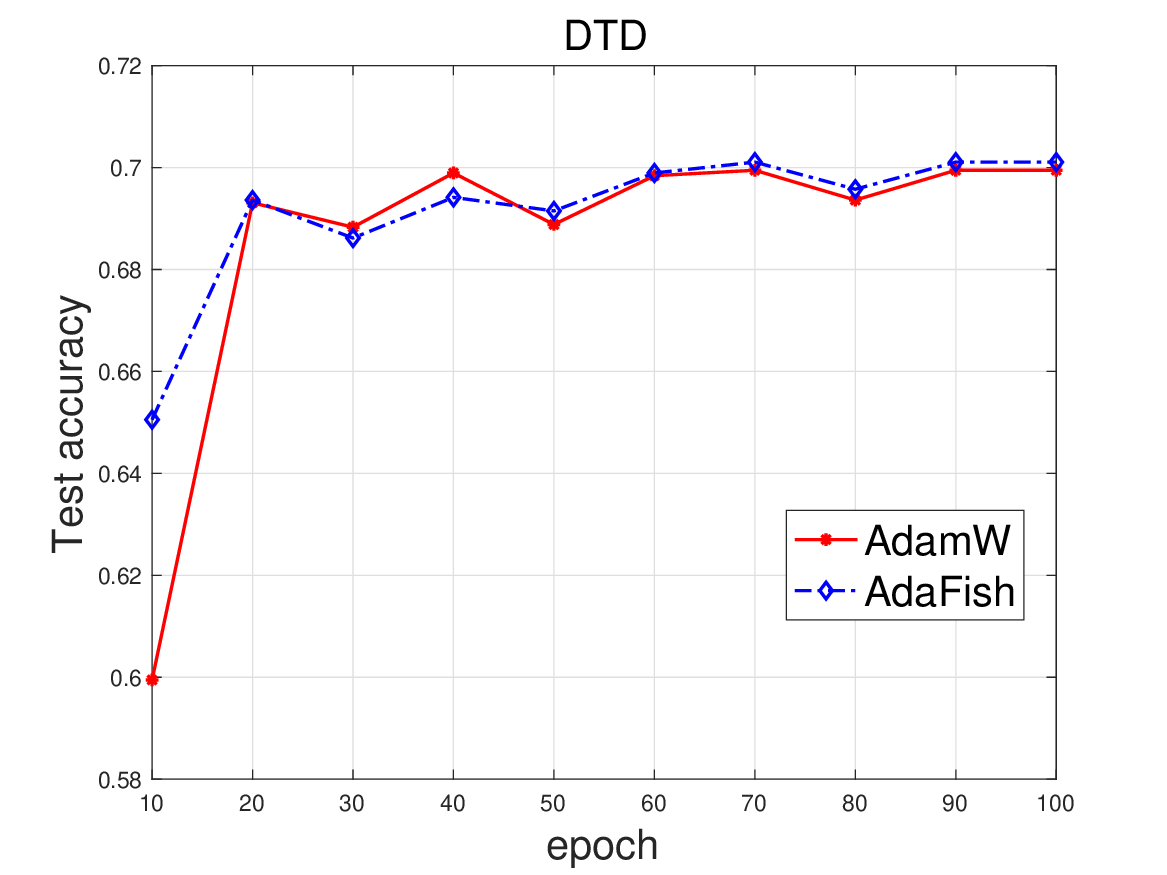}
    \caption{Training losses and testing accuracies on Cifar, Caltech, Oxford\_flowers, and DTD from Vtab-1k. AdaFish shows over 2x faster convergence in training losses and superior test accuracy, indicating enhanced generalization performance.}
    \label{fig:vtab-cifar}
\end{figure*}

\section{Numerical results}
In this section, we test two tasks: image classification and a task from natural language processing.
\subsection{Image classification}
\subsubsection{Dataset}
We use the VTAB-1K dataset \citep{zhai2019large}, which is a vision task adaptation benchmark dataset. It consists of 19 vision datasets and each of them has 1000 training examples. We choose five datasets for illustration: CIFAR, Caltech101, Oxford Flowers, DTD, and Resisc45. 
\subsubsection{Backbone model}
We employ the popular vision transformer \citep{dosovitskiy2020image}, ViT-B/16, as the backbone model. The pretrained weights are obtained by using the ImageNet 21k dataset. 

\subsubsection{Experimental results}
For the considered datasets, we employ LoRA with $r = 4$ for parameter-efficient fine-tuning of the backbone model using the pretrained weights. We compare our proposed AdaFish with the state-of-the-art AdamW method \citep{loshchilov2018decoupled}. For both methods,  we utilize a grid search to find the best parameters for the initial learning rate and weight decay. The initial learning rates for both algorithms are set to $10^{-1}$ and we use weight decay $10^{-1}$ and $10^{-4}$ for AdaFish and AdamW, respectively. The cosine learning rate scheduling \citep{loshchilov2016sgdr} is employed for both algorithms. The other parameters of AdamW are set to the default values. For our AdaFish, we set $\gamma = 2\times 10^{-4}$, $\beta_1 = 0.8$, $\beta_2 = 0.99$ and $\delta = 10^{-15}$. For all tests, we set the batch size to 32 and the maximum number of epochs to 100. 

We present the training losses and testing accuracies for all five datasets in Figures \ref{fig:vtab-cifar} and \ref{fig:vtab-dtd}. A significant performance enhancement with the AdaFish optimizer over AdamW is observed. Notably, AdaFish achieves superior training loss with only half the number of epochs (50 epochs), demonstrating a more efficient training process. Regarding testing accuracies, AdaFish consistently outperforms AdamW, particularly evident after the completion of 50 epochs. This suggests that AdaFish not only accelerates the training process but also enhances model generalization capabilities. Therefore, we see that the use of the adaptive Fisher information gives a better estimation of the second-order information and leads to fast convergence. 

\begin{figure*}[htbp]
    \centering
    \includegraphics[width=0.34\textwidth]{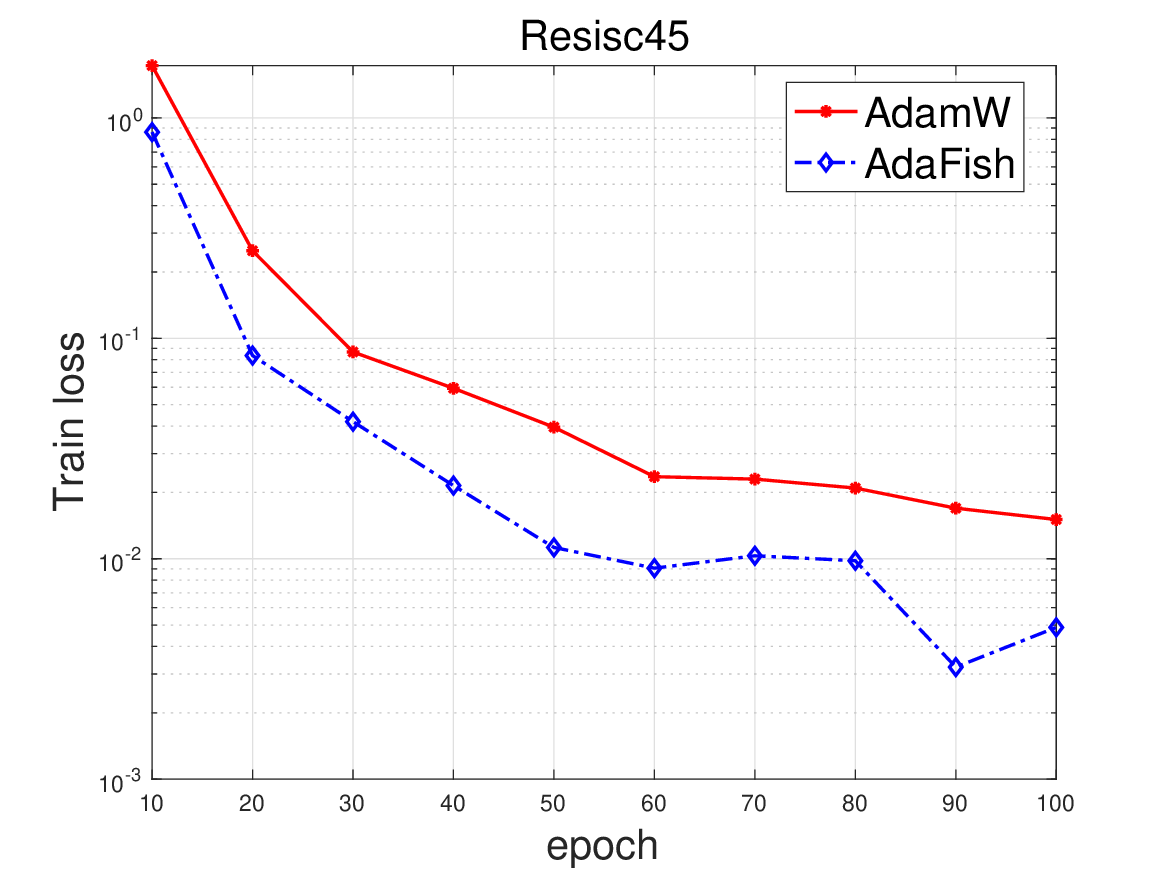}
    \includegraphics[width=0.34\textwidth]{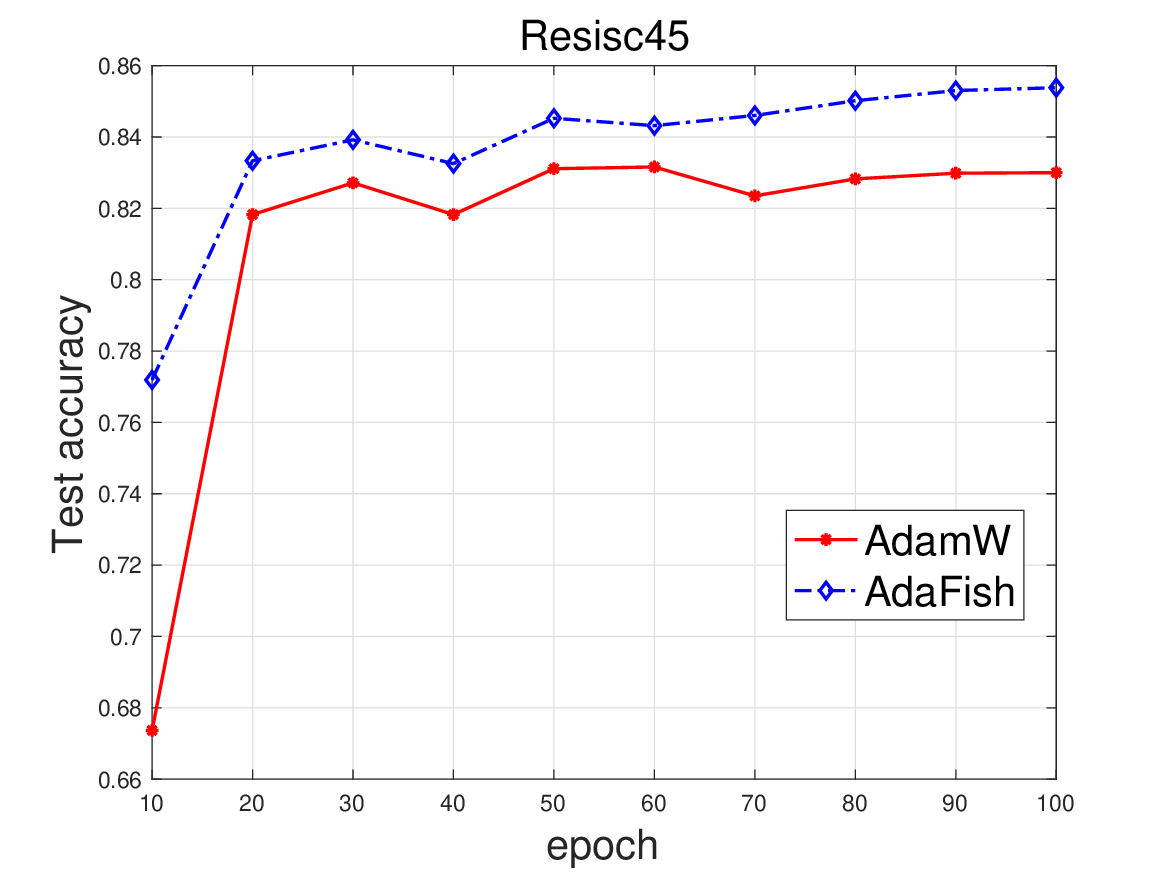}\\
    \caption{Training losses and testing accuracies on Resisc45 from Vtab-1k. AdaFish shows over 2x faster convergence in training losses and superior test accuracy, indicating enhanced generalization performance.}
    \label{fig:vtab-dtd}
\end{figure*}

\subsection{Natural language processing}
In addition to the fine-tuning task on the vision dataset presented above, we also examine our algorithm on a language processing task in the medical domain.
\subsubsection{Dataset}
We use the MIMIC-CXR Database \citep{johnson2019mimic}, which includes chest radiographs in DICOM format with free-text radiology reports.

\subsubsection{Backbone model}
We adopt the LoRA fine-tuning setting from the Litgpt framework \citep{lit-gpt-2023}, which includes state-of-the-art open-source large language models. 

\subsubsection{Experimental results}
The rank $r$ in LoRA is set to 8. For AdaFish and AdamW, we use the same hyperparameters as in the above subsection, except that the initial learning rate is set to $3 \times 10^{-4}$. The batch size is set to 128. 

To assess the quality of the generated summaries, we used the rouge scoring algorithm \citep{lin2004rouge}, which includes R-1, R-2, R-3, R-4, R-5, and R-L, are a series of metrics in natural language processing that evaluate text similarity: R-1 measures unigram overlap, R-2 through R-5 assess bigram to five-gram overlaps respectively for increasing textual detail, and R-L evaluates the longest common subsequence to gauge overall structural similarity. The inputs to our model are clinical findings extracted from radiology reports in the MIMIC dataset. The model then summarizes the findings and reports the most significant observations it identifies, along with possible causes for those findings.

The training losses and the rouge scores on the testing set are listed in Figure \ref{fig:nlp-loss} and Table \ref{tab:rouge_scores}, respectively. Regarding the training losses, compared with AdamW our AdamFish has quite similar performance in the early stage, then converges much faster in the middle stage, and return better loss values in the end. For the rouge scores, higher value indicates better performance. It can be seen that our AdaFish achieves significant improvements over AdamW across all scores. These promising results show the potential of AdaFish in natural language processing tasks. 

\begin{figure}
    \centering
    \includegraphics[width=0.4\textwidth]{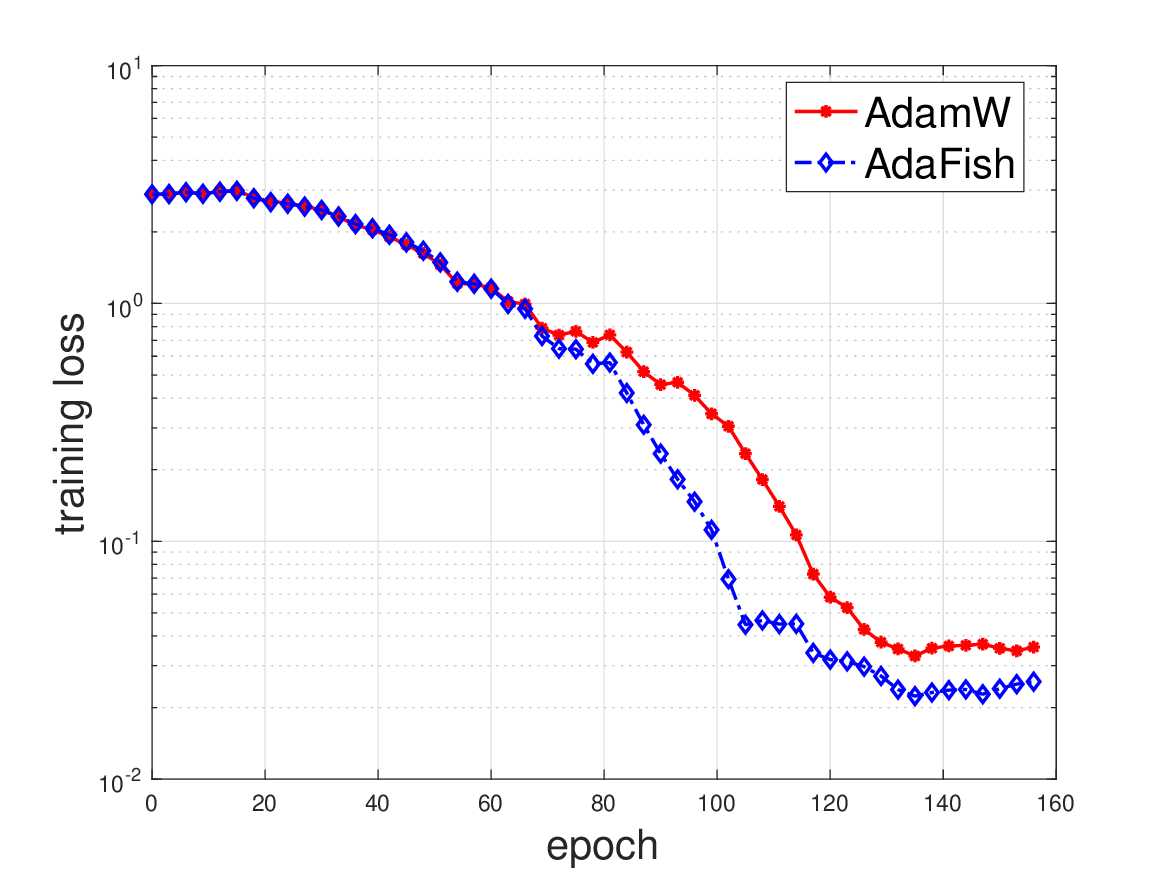}
    \caption{Testing rouge scores obtained by AdamW and AdaFish.}
    \label{fig:nlp-loss}
\end{figure}

\begin{table}[h]
\caption{Comparison of rouge scores for AdamW and AdaFish.}
\centering
\setlength{\tabcolsep}{2pt}
\begin{tabular}{|l|c|c|c|c|c|c|}
\hline
{Method} & {R-1} & {R-2} & {R-3} & {R-4} & {R-5} & {R-L} \\ \hline
AdamW           & 0.1652       & 0.1109       & 0.0807       & 0.0659       & 0.0563       & 0.1412       \\ \hline
AdaFish         & 0.1799       & 0.1193       & 0.0893       & 0.0736       & 0.0615       & 0.1582       \\ \hline
\end{tabular}
\label{tab:rouge_scores}
\end{table}

\section{Conclusion}
We introduce an adaptive Fisher method designed to expedite parameter-efficient fine-tuning within low-rank structures. A pivotal element of our algorithmic framework is the development of a portable Fisher information matrix, which is shown to be equivalent to the Hessian matrix under certain conditions. By leveraging the storage and computational efficiency of this approach, we employ the exponential moving average on the Fisher information matrix to derive an adaptive version. This adaptation allows our AdaFish method to capture more accurate second-order information. We establish global convergence along with iteration/oracle complexity. Our numerical experiments across image classification and language processing tasks underscore the efficacy of the proposed method. Notably, our algorithm holds potential for extension to a broader spectrum of low-rank-based fine-tuning frameworks, which we leave for future exploration.

\newpage
\nocite{langley00}
\newpage

\bibliography{ref}


\newpage
\appendix
\onecolumn

\section{Proof of Theorem \ref{thm:global}}
\begin{proof}
    Due to the similarity between our AdaFish and AdamW, our analysis follows from the argument in \citep{zhou2022towards}. For brevity, we let $v_t= \gamma h_t/(1-\beta_2^t) + \delta I$. Since we have $\left\|g_t \right\|_{\infty} \leq D_g$ and $\delta \leq \| v_t \| \leq$ $D_g^2+\delta =: c_1$. 
    Also we define
    \be u_t:=m_t+\lambda v_t \theta_t, \quad \theta_{t+1}-\theta_t =-\eta  v_t^{-1}(m_t+\lambda v_t \theta_t) =-\eta v_t^{-1} u_t. 
    \ee
    Moreover, we also define $\widetilde{f}_t\left(\theta\right)$ as follows:
    $$
\widetilde{f}_t\left(\theta\right)=f(\theta)+\lambda_t\|\theta\|_{v_t}^2=\mathbb{E}_{\xi}[f(\theta; \xi)]+\lambda_k\|\theta\|_{{v}_t}^2,
$$
where $\lambda_k=\frac{\lambda}{2} \sum_{i=1}^k\left(\frac{1-\mu}{2}\right)^i(k>0)$ and $\lambda_0=0$ in which $\mu=\frac{4(1-\beta_2) D_g^2}{\delta}$.
Then by using the smoothness of $f(\theta; \xi)$, we can obtain
\be \label{eq:desc1}
\begin{aligned}
\widetilde{f}_{t+1}\left(\theta_{t+1}\right) \leq & f\left(\theta_t\right)+\left\langle\nabla f\left(\theta_t\right), \theta_{t+1}-\theta_t\right\rangle+\frac{L}{2}\left\|\theta_{t+1}-\theta_t\right\|^2+\lambda_{t+1}\left\|\theta_{t+1}\right\|_{v_{t+1}}^2 \\
\stackrel{(1)}{\leq} & f\left(\theta_t\right)+\left\langle\nabla f\left(\theta_t\right), \theta_{t+1}-\theta_t\right\rangle+\frac{L}{2}\left\|\theta_{t+1}-\theta_t\right\|^2+ \frac{\lambda_{t+1}}{1-\mu}\left\|\theta_{t+1}\right\|_{v_t}^2 \\
\stackrel{(2)}{\leq} &f\left(\theta_t\right)+\lambda_t\left\|\theta_t\right\|_{v_t}^2+\left\langle\nabla f\left(\theta_t\right)+\lambda v_t \theta_t, \theta_{t+1}-\theta_t \right\rangle+\frac{L}{2}\left\|\theta_{t+1}-\theta_t\right\|^2+\frac{\lambda}{2}\left\|\theta_{t+1}-\theta_t \right\|_{v_t}^2 \\
= & \widetilde{f}_t\left(\theta_t\right)-\eta\left\langle\nabla f\left(\theta_t\right)+\lambda v_t \theta_t, v_t^{-1} u_t\right\rangle+\frac{L \eta^2}{2}\left\|v_t^{-1} u_t\right\|^2+\frac{\lambda \eta^2}{2}\left\|v_t^{-1} u_t\right\|_{v_t}^2 \\
= & \widetilde{f}_t\left(\theta_t\right) + \frac{1}{2}\left\| \sqrt{\eta} v_t^{-
\frac{1}{2}} (\nabla f(\theta_t) + \lambda v_t \theta_t - u_t) \right\|^2 - \frac{1}{2} \| \sqrt{\eta} v_t^{-\frac{1}{2}} (\nabla f(\theta_t) + \lambda v_t \theta_t) \| \\
& - \frac{1}{2} \|\sqrt{\eta} v_t^{-\frac{1}{2}} u_t \|^2 +\frac{L \eta^2}{2}\left\|v_t^{-1} u_t\right\|^2+\frac{\lambda \eta^2}{2}\left\|v_t^{-1} u_t\right\|_{v_t}^2 \\
\leq & \widetilde{f}_t\left(\theta_t\right) + \frac{\eta}{2\delta} \| \nabla f(\theta_t) - m_t\|^2 - \frac{\eta}{2c_1}\| \nabla f(\theta_t) + \lambda v_t \theta_t \|^2 - \left[ \frac{\eta}{2c_1} - \frac{L\eta^2}{2\delta^2} - \frac{\lambda\eta^2}{2\delta} \right] \|u_t\|^2 \\
\leq & \widetilde{f}_t\left(\theta_t\right) + \frac{\eta}{2\delta} \| \nabla f(\theta_t) - m_t\|^2 - \frac{\eta}{2c_1}\| \nabla f(\theta_t) + \lambda v_t \theta_t \|^2 - \frac{\eta}{4c_1} \|u_t\|^2.
\end{aligned}
\ee
where the second inequality is from $\| v_{t+1}^{-1} v_t \| = \| (I + (1-\beta_2)(g_tg_t^\top - h_k) v_t^{-1})^{-1} \| \in [1-\mu, 1+ \mu]$ with $\mu:= \frac{4(1-\beta_2)D_g^2}{\delta}$, and the third inequalty holds because
\[ \begin{aligned}
    \frac{\lambda_{t+1}}{1-\mu}\| \theta_{t+1} \|_{v_t}^2 & = \frac{\lambda_{t+1}}{1-\mu} \left( \lambda_t \|\theta_t\|^2_{v_t} + 2 \iprod{\theta_{t+1} - \theta_t}{\theta_t} + \|\theta_{t+1} - \theta_t\|^2_{v_t} \right) \\
    & \leq \lambda_k \|\theta_t\|^2_{v_t} + \lambda \iprod{\theta_{t+1 - \theta_t}}{\theta_t} + \frac{\lambda}{2} \|\theta_{t+1} - \theta_t\|^2_{v_t},
\end{aligned}  \]
and (3) is from $\eta \leq \frac{\delta^2}{2c_1(L + \lambda \delta)}$ such that $\frac{\eta}{4c_1} \geq \frac{L\eta^2}{2\delta^2} + \frac{\lambda \eta^2}{2\delta}$. 
Note that
\be \label{eq:diff}
\begin{aligned}
& \mathbb{E}\left[\left\| m_t-\nabla f\left(\theta_t\right)\right\|^2\right] \\
\leq & \beta_1 \mathbb{E}\left[\left\|m_{t-1}-\nabla f\left(\theta_{t-1}\right)\right\|^2\right]+\frac{\beta_1^2 L^2}{ 1- \beta_1} \mathbb{E}\left[\left\|\theta_t-\theta_{t-1}\right\|^2\right]+\frac{ (1- \beta_1)^2 \sigma^2}{b} \\
\leq & \beta_1 \mathbb{E}\left[\left\|m_{t-1}-\nabla f\left(\theta_{t-1}\right)\right\|^2\right]+\frac{\beta_1^2 L^2 \eta^2}{(1 - \beta_1) \delta^2} \mathbb{E}\left[\left\| u_k\right\|^2\right]+\frac{(1 - {\beta}_1)^2 \sigma^2}{b}.
\end{aligned}
\ee
By adding \eqref{eq:desc1} and $\alpha \times$ \eqref{eq:diff} yields
\[ \begin{aligned}
    & \tilde{f}_{t+1}(\theta_{t+1}) + \alpha \mathbb{E}\left[\left\| m_{t+1}-\nabla f\left(\theta_{t+1}\right)\right\|^2\right] \\
    \leq & \widetilde{f}_t\left(\theta_t\right) - \frac{\eta}{2c_1}\| \nabla f(\theta_t) + \lambda v_t \theta_t \|^2  + \left[ \beta_1 \alpha + \frac{\eta}{2\delta} \right] \mathbb{E} \left[\| \nabla f(\theta_t) - m_t\|^2 \right] \\
    & - \left[ \frac{\eta}{4c_1} - \frac{\alpha \beta_1^2 L^2 \eta^2}{(1 - \beta_1) \delta^2} \right] \mathbb{E} \left[ \|u_t\|^2 \right] + \frac{\alpha(1 - {\beta}_1)^2 \sigma^2}{b}.
\end{aligned} \]
Let us denote $\alpha = \frac{\eta}{2\delta (1 -\beta_1)}$ and $G(\theta_{t+1}) = \tilde{f}_{t+1}(\theta_{t+1}) + \frac{\eta}{2\delta(1-\beta_1)} \mathbb{E} \left[\| \nabla f(\theta_t) - m_t\|^2 \right]$. Then,
\[  \begin{aligned}
    G(\theta_{t+1}) & \leq G(\theta_{t}) - \frac{\eta}{2c_1} \mathbb{E} \|\nabla f(\theta_t) + \lambda v_t \theta_t  \|^2 - \frac{\eta}{4c_1} \left[1 - \frac{2 c_1 \beta_1^2 L^2 \eta^2}{(1 - \beta_1)^2 \delta^3}  \right] \mathbb{E}\left[\|u_k\|^2\right] + \frac{\eta (1-\beta_1)\sigma^2  }{2\delta b} \\
    & \leq G(\theta_{t}) - - \frac{\eta}{2c_1} \mathbb{E} \|\nabla f(\theta_t) + \lambda v_t \theta_t  \|^2 - \frac{\eta}{8c_1} \mathbb{E}\left[\|u_k\|^2\right] + \frac{\eta (1-\beta_1)\sigma^2  }{2\delta b},
\end{aligned}
\]
where the second inequality holds due to $\eta \leq \frac{(1-\beta_1)c_1}{2\beta_1 L} \sqrt{\frac{\delta}{c_1}}$. 
Summing the above inequality from $t=0$ to $t=T-1$ leads to
\[ \begin{aligned}
    \frac{1}{T} \sum_{t=0}^{T-1} \mathbb{E} \left[\|\nabla f(\theta_t) + \lambda v_t \theta_t  \|^2 + \frac{1}{4} \|u_k\|^2 \right]  & \leq \frac{2c_1}{\eta T}[G(\theta_0) - G(\theta_T)] + \frac{c_1 (1-\beta_1)\sigma^2}{\delta b} \\
    & \leq \frac{2c_1\Delta}{\eta T} + \frac{c_2 \sigma^2}{c_1 (1-\beta_1) bT} + \frac{c_1 (1-\beta_1)\sigma^2}{\delta b} \\
    &\leq \epsilon^2,
\end{aligned}
 \]
 where we set $T \geq \max\left( \frac{6c_1\Delta}{\eta\epsilon^2}, \frac{3c_1 \sigma^2}{\delta (1-\beta_1)b \epsilon^2} \right)$ and $1- \beta_1 \leq \frac{\delta b\epsilon^2}{3c_1 \sigma^2}$, and 
 \[ 
 \begin{aligned}
& G\left(\boldsymbol{x}_0\right)-G\left(\boldsymbol{x}_T\right) \\
= & \widetilde{f}_0\left(\boldsymbol{x}_0\right)+\frac{\eta}{2 \delta (1 - \beta_1)} \mathbb{E}\left[\left\| m_0-\nabla f\left(\theta_0\right)\right\|^2\right]-\widetilde{f}_T\left(\theta_T\right)-\frac{\eta}{2 \delta (1-\beta_1)} \mathbb{E}\left[\left\|m_T-\nabla f\left(\theta_T\right)\right\|^2\right] \\
= & f\left(\theta_0\right)+\frac{\eta}{2 \delta (1-{\beta}_1)} \mathbb{E}\left[\left\|{m}_0-\nabla f\left(\theta_0\right)\right\|^2\right]-f\left(\theta_T\right)-\lambda_T\left\|\theta_T\right\|_{{v}_T}-\frac{\eta}{2 \delta (1-{\beta}_1)} \mathbb{E}\left[\left\|{m}_T-\nabla f\left(\theta_T\right)\right\|^2\right] \\
\leq & f\left(\theta_0\right)+\frac{\eta}{2 \delta (1 - {\beta}_1)} \mathbb{E}\left[\left\| m_0-\nabla f\left(\theta_0\right)\right\|^2\right]-f\left(\theta_T\right) \\
\leq & \Delta+\frac{\eta}{2 \delta (1 - \beta_1)} \mathbb{E}\left[\left\|m_0-\nabla f\left(\theta_0\right)\right\|^2\right] \\
\leq & \Delta+\frac{\eta \sigma^2}{2 \delta (1- \beta_1) b}. 
\end{aligned}
 \]
 This result directly bounds
 \[ \frac{1}{T}\sum_{t=0}^T \| v_t(x_t - x_{t+1}) \|^2 = \frac{\eta^2}{T} \sum_{t=0}^{T-1} \|m_k + \lambda v_t x_t\|^2 \leq \frac{\eta^2}{T} \sum_{t=0}^{T-1} \|u_t\|^2 \leq 4 \eta^2 \epsilon^2 \]
 and 
\[ \frac{1}{T}\sum_{t=0}^{T-1} \| \theta_t - \theta_{t+1} \|^2 \leq \frac{4 \eta^2 \epsilon^2}{\delta^2}.  \]
For all hyper-parameters, we put their constrains together:
$$
1- {\beta}_1 \leq \frac{c_1 b \epsilon^2}{3 c_1 \sigma^2}
$$
where $\delta \leq\left\| {v}_t\right\| \leq D_g^2+\delta =c_1=\mathcal{O}\left(D_g^{2}\right)$. For $\eta$, it should satisfy
$$
\eta \leq \frac{(1 - {\beta}_1) \delta}{2\beta_1 L} \sqrt{\frac{\delta}{c_1}} \leq \frac{\delta b \epsilon^2}{3 c_1 \sigma^2} \frac{\delta}{2 L} \sqrt{\frac{\delta}{c_1}}=\frac{\delta^2 b \epsilon^2}{6 c_1 \sigma^2 L} \sqrt{\frac{\delta}{c_1}},
$$
where $\delta$ is often much smaller than one, and $1 - \beta_1$ is very small. For $T$, we have
$$
\begin{aligned}
T & \geq \max \left(\frac{6 c_1 \Delta}{\eta \epsilon^2}, \frac{3 c_1 \sigma^2}{\delta (1-\beta_1) b \epsilon^2}\right)=\mathcal{O}\left(\max \left(\frac{6 c_1 \Delta}{\epsilon^2} \frac{6 c_1 \sigma^2 L}{\delta^2 b \epsilon^2} \sqrt{\frac{c_1}{\delta}}, \frac{3 c_1 \sigma^2}{\delta b \epsilon^2} \frac{3 c_1 \sigma^2}{\delta b \epsilon^2}\right)\right) \\
& =\mathcal{O}\left(\max \left(\frac{36 c_1^{2.5} \Delta \sigma^2 L}{\delta^{2.5} b \epsilon^4}, \frac{9 c_1^2 \sigma^4}{\delta^2 b^2 \epsilon^4}\right)\right)=\mathcal{O}\left(\max \left(\frac{36 D_g^{2.5} \Delta \sigma^2 L}{\delta^{1.25} b \epsilon^4}, \frac{9 D_g^2 \sigma^4}{\delta b^2 \epsilon^4}\right)\right) .
\end{aligned}
$$

Now we compute the stochastic gradient complexity. For $T$ iterations, the complexity is
$$
\mathcal{O}(T b)=\mathcal{O}\left(\max \left(\frac{36 c_1^{2.5} \Delta \sigma^2 L}{\delta^{2.5} \epsilon^4}, \frac{9 c_1^2 \sigma^4}{\delta^2 b \epsilon^4}\right)\right)=\mathcal{O}\left(\max \left(\frac{36 D_g^{2.5} \Delta \sigma^2 L}{\delta^{1.25} \epsilon^4}, \frac{9 D_g^2 \sigma^4}{\delta b \epsilon^4}\right)\right)
$$

The proof is completed.

\end{proof}
\end{document}